\documentclass{sig-alternate-per}

\usepackage{cite}
\usepackage{amsmath,amssymb,amsfonts}

\usepackage{graphicx}
\usepackage{textcomp}
\usepackage{xcolor}
\def\BibTeX{{\rm B\kern-.05em{\sc i\kern-.025em b}\kern-.08em
    T\kern-.1667em\lower.7ex\hbox{E}\kern-.125emX}}

\usepackage{subcaption} 
\usepackage{amsmath}
\usepackage{amsfonts}
\usepackage{mathrsfs}
\usepackage{amssymb}
\usepackage{dsfont}
\usepackage{times}
\usepackage{multirow}
\usepackage[none]{hyphenat}
\usepackage{float}
\usepackage{subfig}
\usepackage{iftex}
\usepackage{algorithm}
\usepackage{algpseudocode}
\usepackage{booktabs}
\usepackage{cleveref}

\newtheorem{corollary}{Corollary}

\newtheorem{remark}{\bf Remark}
\newtheorem{assumption*}{Assumption}
\newtheorem{stdassumption*}{Standing Assumption}

\newtheorem{prop}{{Proposition}}

\newtheorem{definition*}{{Definition}}

\usepackage{microtype} 

\begin{document}
%

\title{Toward Reliable Design of LLM-Enabled Agentic Workflows: Optimizing Latency-Reliability-Cost Tradeoffs}

%
%
%
%
%

\numberofauthors{2} 
%
\author{
%
\alignauthor Ya-Ting Yang\\
       \affaddr{New York University}\\
       \affaddr{Brooklyn, NY, USA}\\
       \email{yy4348@nyu.edu}
\alignauthor Quanyan Zhu\\
 \affaddr{New York University}\\
       \affaddr{Brooklyn, NY, USA}\\
       \email{qz494@nyu.edu}
}

\maketitle
\begin{abstract}
Modern AI systems increasingly rely on workflows composed of multiple interacting agents, some powered by large language models (LLMs) and others by conventional computational modules. This paper analyzes the fundamental tradeoffs between latency, reliability, and cost in LLM-enabled agentic workflows. We introduce performance models for both LLM and non-LLM agents that capture the relationship between computational effort and output quality, incorporating the impact of reasoning and output tokens for LLM agents using a parametric exponential reliability function. Then, we study the design of sequential workflows under latency and cost constraints. 
Main results include a water-filling token allocation policy and characterizations of optimal workflow reliability in terms of shadow prices. 
\end{abstract}

\section{Introduction}

With recent advances in artificial intelligence (AI), particularly large language models (LLMs), modern systems increasingly rely on workflows composed of multiple interacting agents. In these systems, agents collaborate to solve complex tasks by exchanging intermediate information and progressively refining their outputs \cite{acharya2025agentic,yang2026internet}. These agents may include LLM-based components as well as conventional computational modules, such as database queries, optimization solvers, and verification tools. 
%
As these workflows grow in scale and complexity, their deployment also becomes increasingly resource-intensive. This raises fundamental challenges in how to configure LLM-enabled systems efficiently and reliably. In practice, system design is often guided by heuristics or costly trial-and-error. Hence, there is a need for mathematical and principle frameworks to analyze and optimize agentic LLM workflows.

From a systems perspective, workflow performance can be characterized by three key quantities: latency, reliability, and cost. Optimizing these performance metrics requires understanding how individual agent characteristics combine to determine overall workflow behavior.
This study develops performance models for LLM-enabled agentic workflows and derives optimal resource allocation policies under latency and cost constraints. We introduce parametric models that capture the relationship between computational effort and output quality for both LLM and non-LLM agents. Then, we focus on sequential workflows, formulate the design problem as maximizing reliability subject to resource budgets, and derive closed-form allocation policies using convex optimization method \cite{boyd2004convex}. Our results provide designers with principled methods for configuring LLM-based systems to achieve desired performance tradeoffs.

\section{LLM Agentic Workflows}

An LLM agentic workflow can be represented as a directed computation graph where nodes correspond to agents and edges represent information flow. Tasks enter the workflow as input messages, propagate through agents, and eventually produce outputs.

\subsection{LLM Agent Performance Metrics}

In practice, the execution of an LLM agent involves two distinct computational stages. The agent first performs internal reasoning to construct a solution to the task. Then, it generates a sequence of tokens that communicates the solution to downstream agents or to the user. Both stages consume computational resources and therefore influence latency, reliability, and cost.
Considering an LLM agent indexed by $j$, we define two key resource variables: $X_j$ denotes the number of ``reasoning tokens'' generated internally during the reasoning process, and $L_j$ denotes the number of ``output tokens'' for the final response. Both quantities are measured in units of tokens. The variable $X_j$ represents the internal computational effort, while $L_j$ determines the length of the communicated response.

\textbf{Latency:} The latency of an LLM agent depends on network conditions, system scheduling, and variability in model execution. We model the latency of agent $j$ as a random variable $T_j$ measured in units of time (e.g., seconds). The latency consists of two components: infrastructure delay and model inference. The infrastructure delay is denoted by $\tau_j$, which captures network transmission and orchestration overhead. This quantity is treated as a random variable with mean $\bar{\tau}_j = \mathbb{E}[\tau_j]$.
During model inference, processing rates differ in stages \cite{kwon2023efficient}. We denote by $\lambda_{\text{th}}$ the average processing rate for reasoning tokens and by $\lambda_{\text{gen}}$ the average generation rate for output tokens. 
%
The total latency of agent $j$ is $T_j = \tau_j + \frac{X_j}{\lambda_{\text{th}}} + \frac{L_j}{\lambda_{\text{gen}}}$,
where the first term captures the stochastic infrastructure delay, and the remaining terms represent deterministic computation time.
Taking expectations yields the mean latency
\begin{equation}
\bar{T}_j = \bar{\tau}_j + (X_j/\lambda_{\text{th}}) + (L_j/\lambda_{\text{gen}}),
\label{eq:mean_latency}
\end{equation}
which we use as the latency metric in subsequent analysis. 

\textbf{Reliability:} The reliability of agent $j$ is defined as the probability that the output produced by the agent is acceptable. 
%
We define the reliability function $\rho_j : \mathbb{R}^2_+ \to [0,1]$ by
$\rho_j(X_j, L_j) = \Pr(\text{output of } j \text{ is acceptable} | X_j, L_j)$.
This function captures how the quality of the agent's output depends on the computational resources allocated to reasoning and generation.
Typically, increasing the reasoning effort $X_j$ improves the quality of the internal solution; increasing the output length $L_j$ may enhance the clarity, explanation, or verification of the result \cite{wang2022self}. To capture diminishing marginal returns in both dimensions, we adopt the following parametric form:
\begin{equation}
\rho_j(X_j, L_j) = \big(1 - e^{-\alpha_j X_j}\big)\big(1 - e^{-\beta_j L_j}\big),
\label{eq:reliability}
\end{equation}
where $\alpha_j > 0$ and $\beta_j > 0$ are agent-specific parameters in units of $(\text{tokens})^{-1}$. This choice ensures that the arguments of the exponential functions are dimensionless and that $\rho_j(X_j, L_j) \in [0,1]$. 

\textbf{Cost:} There are two distinct cost notions that arise in LLM-based systems. First, the user-visible cost is associated with the number of output tokens generated. We denote by $c_{\text{tok}} > 0$ the price per output token, measured in units of currency per token. The resulting user cost of agent $j$ is $C^u_j(L_j) = c_{\text{tok}} L_j$.
Second, the system provider incurs computational cost when processing tokens. Since both reasoning tokens and output tokens must be processed by the model, the total token workload equals $X_j + L_j$. Let $c_{\text{comp}} > 0$ denote the computational cost per processed token (in units of currency per token). The corresponding computational cost is
$C^c_j(X_j, L_j) = c_{\text{comp}}(X_j + L_j)$.

\begin{remark}
    According to OpenAI \cite{kaplan2020scaling}, the floating point of operations (FLOPs) per token for the transformer-based model is: $2\beta + 2n^{\text{layer}} n^{\text{ctx}} n^{\text{attn}}$, where $\beta$ is the number of parameters, $n^{\text{layer}}$, $n^{\text{ctx}}$, and $n^{\text{attn}}$ are the number of layers, maximum input tokens, and the dimension of the attention output, respectively. Hence, $c_{\text{comp}}$ can be expressed as: $c_{\text{comp}}=c_e \cdot \big( 2\beta + 2n^{\text{layer}} n^{\text{ctx}} n^{\text{attn}}\big)$, where $c_e \in \mathbb{R}_{+}$ denotes the monetary computational (energy) cost per FLOP. Note that $c_{\text{comp}}$ and $c_{\text{tok}}$ may also be related.
\end{remark}

\subsection{Non-LLM Agent Performance Metrics}

Non-LLM agents typically perform deterministic or algorithmic computations on structured inputs and do not generate token sequences.
In many cases, the latency $T_i$ of agent $i$ can be modeled using queueing-theoretic abstractions. For instance, if the service time follows an exponential distribution with rate $\lambda_i$ (measured in $(\text{time})^{-1}$), the expected latency is $\mathbb{E}[T_i] = 1/ \lambda_i$.
The reliability of the agent is defined as the probability that the agent produces a correct output given a valid input, denoted by $\rho_i \in [0,1]$. 
The pair $(T_i, \rho_i)$ therefore provides a compact representation of the latency and reliability characteristics of a non-LLM agent. 

\subsection{Workflow Aggregation}
\label{subsec:workflow_config}

Agentic workflows combine heterogeneous agents in different structural configurations. Let $k$ denote an arbitrary agent with latency $T_k$ and reliability $\rho_k \in [0,1]$. The performance of a workflow depends not only on the characteristics of individual agents but also on the structural configuration through which agents interact. We consider three common configurations: sequential composition, parallel composition, and feedback (iterative) loops.

\textbf{Sequential:} In a sequential configuration, the output of one agent becomes the input to the next. 
Consider a pipeline of $n$ agents indexed by $k = 1, \ldots, n$. Since each stage must complete before the next stage begins, the expected workflow latency is $\mathbb{E}[T_{\text{seq}}] = \sum_{k=1}^{n} \mathbb{E}[T_k]$. The workflow produces a correct result only if all agents succeed, yielding the reliability $R_{\text{seq}} = \prod_{k=1}^{n} \rho_k$.

\textbf{Parallel:} In a parallel configuration, multiple agents process the same input simultaneously. The latency of the parallel stage is determined by the critical path. For two agents $i$ and $j$, the expected latency is $\mathbb{E}[T_{i \| j}] = \mathbb{E}[\max(T_i, T_j)]$. Reliability, however, depends on how the outputs are aggregated. If both outputs are required for the workflow to succeed, the reliability is $R_{i \| j} = \rho_i \rho_j$. In contrast, if the workflow succeeds whenever at least one branch produces a correct result, the reliability becomes $R_{i \| j} = 1 - (1 - \rho_i)(1 - \rho_j)$. These two cases correspond to conjunctive and redundant parallel execution patterns, respectively.

\textbf{Feedback:} Many agentic workflows employ iterative reasoning in which agents repeatedly exchange information to refine a solution. Such structures arise in debate systems, verification loops, and iterative plannings. Consider a feedback loop in which agents $i$ and $j$ execute in the sequence $i \to j$ for $K$ iterations. Each iteration has latency $T_i + T_j$, so the expected latency after $K$ iterations is
$\mathbb{E}[T_{\text{fb}}] = K(\mathbb{E}[T_i] + \mathbb{E}[T_j])$. If both agents must produce correct outputs in each iteration, the reliability becomes $R_{\text{fb}} = (\rho_i \rho_j)^K$.

\section{Performance Optimization}

We now turn to the design problem of optimizing agentic workflows. In this work, we focus on the case of sequential composition. It is worth noting that studying sequential workflows is not a loss of generality: more complex architectures constructed from parallel or feedback structures can be viewed as sequential compositions of composite modules, where each module corresponds to a sub-workflow whose latency and reliability are determined by the aggregation rules introduced in Section~\ref{subsec:workflow_config}.

Consider a sequential workflow consisting of agents indexed by $k = 1, \ldots, n$. Let $\mathcal{A}_{\text{LLM}}$ denote the set of LLM agents and $\mathcal{A}_{\text{NLLM}}$ denote the set of non-LLM agents. For non-LLM agents $i \in \mathcal{A}_{\text{NLLM}}$, the expected latency $\mathbb{E}[T_i]$ and reliability $\rho_i$ are fixed parameters determined by the system implementation. For LLM agents $j \in \mathcal{A}_{\text{LLM}}$, the response length $L_j$ is a decision variable chosen by the system designer. For LLM agents, expected latency $\mathbb{E}[T_j(L_j)]$ follows~\eqref{eq:mean_latency} and the reliability $\rho_j(L_j)$ follows~\eqref{eq:reliability}.

Since agents operate sequentially, the expected latency of the workflow equals the sum of individual expected latencies, $\mathbb{E}[T_{\text{wf}}(L)] = \sum_{i \in \mathcal{A}_{\text{NLLM}}} \mathbb{E}[T_i] + \sum_{j \in \mathcal{A}_{\text{LLM}}} \mathbb{E}[T_j(L_j)]$.
Then, assuming independent failures across agents, the workflow reliability equals the product of the reliabilities of all agents in the pipeline.
For non-LLM agents $i \in \mathcal{A}_{\text{NLLM}}$, the reliability $\rho_i \in [0,1]$ is a fixed parameter determined by the system implementation. For LLM agents $j \in \mathcal{A}_{\text{LLM}}$, reliability depends on both the reasoning effort $X_j$ and the output length $L_j$. The overall workflow reliability is therefore
$R_{\text{wf}}(X, L) = \prod_{i \in \mathcal{A}_{\text{NLLM}}} \rho_i \prod_{j \in \mathcal{A}_{\text{LLM}}} \rho_j(X_j, L_j)$.

In the subsequent analysis, we treat the reasoning allocations $X_j$ as fixed parameters determined by the internal reasoning policy of each agent. We therefore introduce the reduced reliability function
$\hat{\rho}_j(L_j) := \rho_j(X_j, L_j)$,
which depends only on the response length $L_j$. Under this assumption, the workflow reliability becomes
$R_{\text{wf}}(L) = \prod_{i \in \mathcal{A}_{\text{NLLM}}} \rho_i \prod_{j \in \mathcal{A}_{\text{LLM}}} \hat{\rho}_j(L_j)$.

\subsection{Latency-Cost-Constrained Design} \label{sec:opt_problem}
Suppose that the workflow must satisfy both a latency constraint $T$ and a user cost budget $C$. The system designer selects the response lengths $L_j$ for the LLM agents to maximize the reliability of the workflow. The resulting design problem can be written as
\begin{align}
\max_{L_j, j \in \mathcal{A}_{\text{LLM}}} & \prod_{j \in \mathcal{A}_{\text{LLM}}} \hat{\rho}_j(L_j) \label{eq:opt_problem} \\
\text{s.t.} & \sum_{j \in \mathcal{A}_{\text{LLM}}} \frac{L_j}{\lambda_{\text{gen}}} \leq T - T_{\text{fixed}}, \sum_{j \in \mathcal{A}_{\text{LLM}}} C^u_j(L_j) \leq C, \label{eq:TC_constraints} \\
& L_j \geq 0, \quad j \in \mathcal{A}_{\text{LLM}},  \label{eq:L_constraints}
\end{align}
where $T_{\text{fixed}} = \sum_{i \in \mathcal{A}_{\text{NLLM}}} \mathbb{E}[T_i] + \sum_{j \in \mathcal{A}_{\text{LLM}}} \big( \bar{\tau}_j + \frac{X_j}{\lambda_{\text{th}}} \big)$
represents the latency that does not depend on the response. 
%
Since the logarithm is a strictly increasing function, maximizing the product in~\eqref{eq:opt_problem} is equivalent to maximizing the sum of logarithms. We have
\begin{equation}
\begin{aligned}
\max_{L_j, j \in \mathcal{A}_{\text{LLM}}} & \sum_{j \in \mathcal{A}_{\text{LLM}}} \log \hat{\rho}_j(L_j) \quad
\text{s.t.}  \text{ constraints } \eqref{eq:TC_constraints}, \eqref{eq:L_constraints}
\end{aligned}\vspace{-1mm}
\label{eq:opt_problem_log}
\end{equation}

\subsection{Optimal Token Allocation}

We now solve the log-equivalent problem~\eqref{eq:opt_problem_log}. We denote by $B_T = T - T_{\text{fixed}}$ the latency budget available for output generation. 
%
Since the reasoning allocations $X_j$ are treated as fixed parameters, the reduced reliability function is
$\hat{\rho}_j(L_j) = \big(1 - e^{-\alpha_j X_j}\big) \big(1 - e^{-\beta_j L_j}\big)$.
Substituting the reliability model in~\eqref{eq:reliability}, the objective function in \eqref{eq:opt_problem_log} becomes
$\sum_{j \in \mathcal{A}_{\text{LLM}}} \log (1 - e^{-\alpha_j X_j}) + \sum_{j \in \mathcal{A}_{\text{LLM}}} \log (1 - e^{-\beta_j L_j})$.
Since $X_j$ is fixed, the first term is constant with respect to the decision variables $L_j$. Consequently, the optimization over response lengths is equivalent to maximizing only the second term. 
\begin{equation}
\begin{aligned}
\max_{L_j, j \in \mathcal{A}_{\text{LLM}}} & \sum_{j \in \mathcal{A}_{\text{LLM}}} \log \big(1 - e^{-\beta_j L_j}\big) \\
\text{s.t.}  & \sum_{j \in \mathcal{A}_{\text{LLM}}} \frac{L_j}{\lambda_{\text{gen}}} \leq B_T,\  \sum_{j \in \mathcal{A}_{\text{LLM}}} c_{\text{tok}} L_j \leq C,\  L_j \geq 0.
\end{aligned}\vspace{-3mm}
\label{eq:opt_problem_simplified}
\end{equation}

\begin{prop}[Water-filling token allocation]
\label{thm:water_filling}
Consider problem~\eqref{eq:opt_problem_simplified} and assume $\beta_j > 0, \forall j \in \mathcal{A}_{\text{LLM}}$. Define the effective token budget as
$B = \min\left\{\lambda_{\text{gen}} B_T, \, C/c_{\text{tok}}\right\}$.
Then the optimal response lengths satisfy the water-filling rule
$L^*_j = \big[ \frac{1}{\beta_j} \log \big(1 + \frac{\beta_j}{\theta}\big) \big]_+$,
where $\theta \geq 0$ is the Lagrange multiplier associated with the effective budget constraint and is chosen so that $\sum_{j \in \mathcal{A}_{\text{LLM}}} L_j^* = B$.
The optimal solution $L^*$ is unique.
\end{prop}

\begin{proof}
The constraints in \eqref{eq:opt_problem_simplified}imply that the feasible set of \eqref{eq:opt_problem_simplified} is equivalent to $\{L_j \geq 0 : \sum_j L_j \leq B\}$.
Then, we apply the Karush-Kuhn-Tucker (KKT) conditions by introducing a multiplier $\theta \geq 0$ for $\sum_j L_j \leq B$ and multipliers $\nu_j \geq 0$ for nonnegativity constraints. The Lagrangian is $\mathcal{L}(L, \theta, \nu) = \sum_j \log(1 - e^{-\beta_j L_j}) - \theta (\sum_j L_j - B_L) + \sum_j \nu_j L_j$.
The KKT stationarity condition gives
$\partial \mathcal{L}/\partial L_j = (\beta_j e^{-\beta_j L_j})/(1 - e^{-\beta_j L_j}) - \theta + \nu_j = 0$.
Consider an agent with $L_j > 0$. Complementary slackness implies $\nu_j = 0$, yielding $\beta_j e^{-\beta_j L_j} / (1 - e^{-\beta_j L_j}) = \theta$. Rearranging this equation gives 
$L_j = \frac{1}{\beta_j} \log(1 + \beta_j / \theta)$. If this value is negative, the nonnegativity constraint becomes active and $L_j^* = 0$. Combining these two cases yields the water-filling rule.
As the objective function is strictly concave and the feasible set is convex, KKT conditions are necessary and sufficient for optimality. The solution is unique.
\end{proof}

\begin{remark}[Shadow Price Interpretation]
\label{rem:shadow_price}
The multiplier $\theta$ in Proposition~\ref{thm:water_filling} admits a natural economic interpretation as the shadow price of the token budget, representing the marginal value of an additional token in terms of log reliability. From the KKT conditions, every agent with $L_j^* > 0$ satisfies $\frac{\partial}{\partial L_j} \log \rho_j(X_j, L_j^*) = \theta$,
which implies that the optimal allocation equalizes the marginal gain in log reliability per token across active agents.
\end{remark}

\begin{corollary}[Optimal workflow reliability]
\label{cor:optimal_reliability}
Let $L^*$ be the optimal token allocation characterized in Proposition~\ref{thm:water_filling}. The corresponding optimal workflow reliability is thus
$R_{\text{wf}}^* = \prod_{j \in \mathcal{A}_{\text{LLM}}} \rho_j(X_j, L_j^*) = \prod_{j \in \mathcal{A}_{\text{LLM}}} \big(1 - e^{-\alpha_j X_j}\big) \frac{\beta_j}{\beta_j + \theta}$,
where $\theta$ is the effective shadow token price.
\end{corollary}

\begin{proof}
The proof follows the model in \eqref{eq:reliability} and the optimal allocation in Proposition~\ref{thm:water_filling} that gives $e^{-\beta_j L_j^*} = \theta / (\beta_j + \theta)$.
\end{proof}




\section{Numerical Experiments}

\begin{figure}
    \centering
    \includegraphics[width=0.95\linewidth]{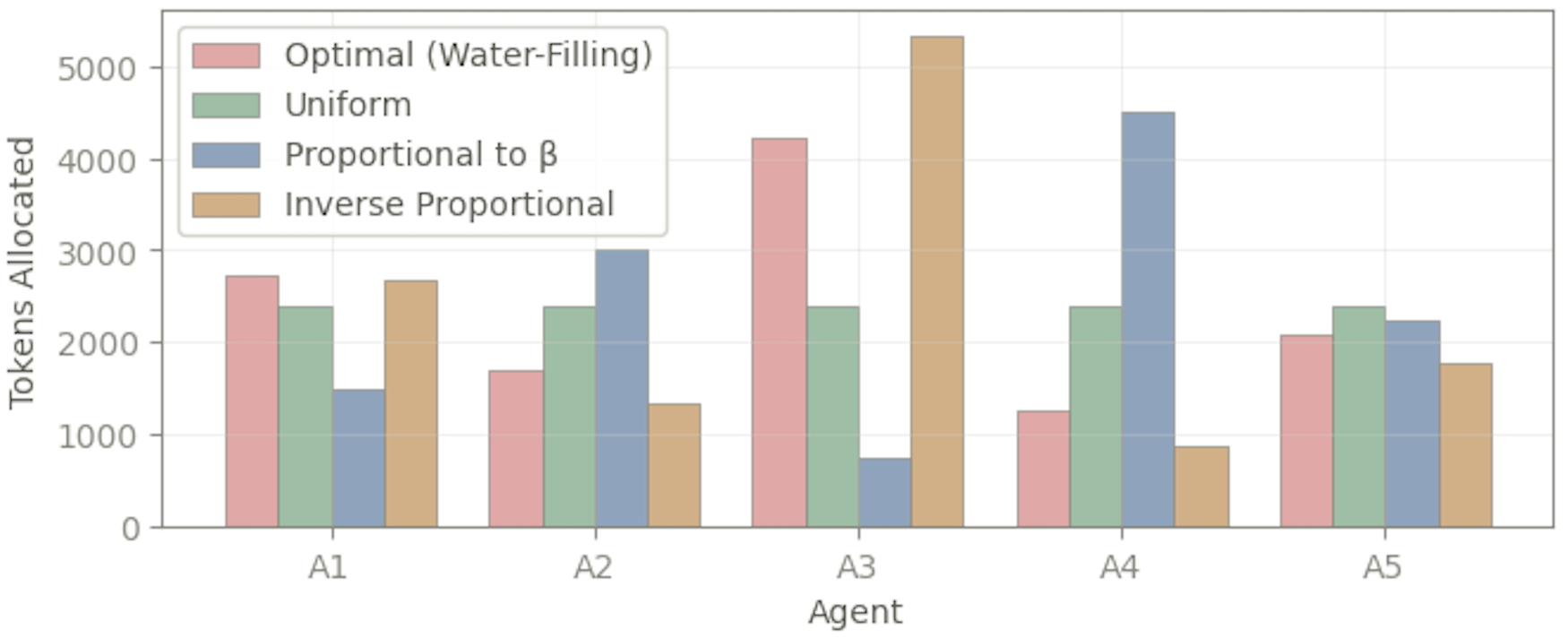}\vspace{-3mm}
    \caption{Output token allocations under different strategies.}
    \label{fig:allocation} \vspace{-3mm}
\end{figure}

\begin{figure}
    \centering
    \includegraphics[width=0.95\linewidth]{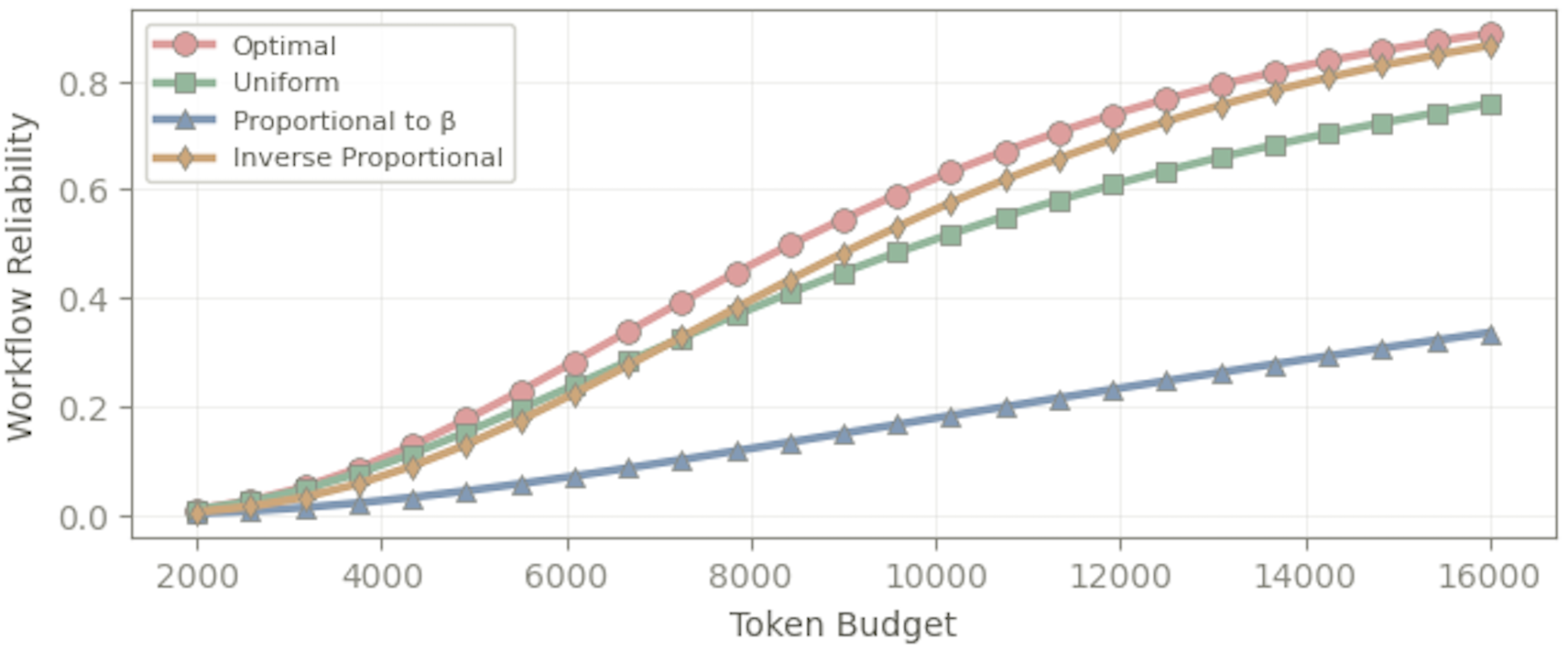}\vspace{-3mm}
    \caption{Workflow reliability under different token budgets.}
    \label{fig:reliability} \vspace{-6mm}
\end{figure}
We illustrate the water-filling allocation policy of Proposition~\ref{thm:water_filling} through a numerical example. Consider a sequential workflow comprising $n = 5$ LLM agents with heterogeneous reliability parameters $\beta_j \in \{0.001, 0.002, 0.0005, 0.003, 0.0015\}$. The total output token budget is $B = 12000$ tokens.
Figure~\ref{fig:allocation} compares the optimal water-filling allocation against three baseline strategies: uniform (equal tokens to all agents), proportional (tokens $\propto \beta_j$), and inverse proportional (tokens $\propto 1/\beta_j$). 
The water-filling rule allocates more tokens to agents with lower values of $\beta_j$. This result reflects that agents with higher $\beta_j$ quickly reach high reliability with fewer tokens, while agents with lower $\beta_j$ require more tokens to achieve comparable levels. 
Figure~\ref{fig:reliability} shows the reliability of the workflow under different token budgets and strategies; the optimal allocation rule consistently dominates all baseline strategies.

\section{Conclusion}
This work proposed a framework for analyzing and optimizing LLM-enabled agentic workflows under latency and cost constraints. We introduced performance models quantifying trade-offs among reliability, latency, and cost, and derived closed-form optimal token allocation policies for sequential workflows. The optimal allocation equalizes marginal gains in log reliability across agents, with the allocation determined by a shadow price that reflects the scarcity of the token budget. Future work includes joint optimization of reasoning and response length, as well as adaptive allocation policies that dynamically adjust token budgets based on intermediate results.

\bibliographystyle{abbrv}
\bibliography{ref}

\end{document}